\documentclass{article}

\usepackage{times}
\usepackage{graphicx} 
\usepackage{subfigure}

\usepackage{todonotes}

\usepackage{amsmath}
\usepackage{amssymb}
\usepackage{amsthm}

\graphicspath{{images/}}

\DeclareMathOperator*{\maximize}{maximize}
\DeclareMathOperator*{\argmin}{argmin}

\DeclareMathOperator*{\minimize}{minimize}
\DeclareMathOperator*{\subjectto}{subject\;to}
 
\DeclareMathOperator*{\diag}{diag}

\newcommand{\dd}{\mathsf{d}}
\newcommand{\R}{\mathbb{R}}

\newtheorem{proposition}{Proposition}

\newcommand*\Let[2]{\State #1 $\gets$ #2}

\definecolor{lightgray}{gray}{0.95} 

\usepackage{hyperref}

\usepackage[accepted]{icml2017}
\icmltitlerunning{Input Convex Neural Networks}

\usepackage{algorithm}
\usepackage{algpseudocode}

\begin{document}

\twocolumn[
\icmltitle{Input Convex Neural Networks}

\icmlsetsymbol{atcmu}{*}

\begin{icmlauthorlist}
\icmlauthor{Brandon Amos}{cmu}
\icmlauthor{Lei Xu}{tsinghua,atcmu}
\icmlauthor{J.~Zico Kolter}{cmu}
\end{icmlauthorlist}
\icmlaffiliation{cmu}{School of Computer Science,
  Carnegie Mellon University. Pittsburgh, PA, USA}
\icmlaffiliation{tsinghua}{Department of Computer Science and Technology,
  Tsinghua University. Beijing, China}

\icmlcorrespondingauthor{Brandon Amos}{bamos@cs.cmu.edu}
\icmlcorrespondingauthor{J.~Zico Kolter}{zkolter@cs.cmu.edu}

\icmlkeywords{deep learning, convex optimization, structured prediction,
  reinforcement learning}

\vskip 0.3in
]

\printAffiliationsAndNotice{\textsuperscript{*}Work done
  while author was at Carnegie Mellon University.}

\begin{abstract}
This paper presents the input convex neural network
architecture. These are scalar-valued (potentially deep) neural
networks with constraints on the network parameters such that the
output of the network is a convex function of (some of) the inputs.
The networks allow for efficient inference via optimization over some
inputs to the network given others, and can be applied to settings
including structured prediction, data imputation, reinforcement
learning, and others.  In this paper we lay the basic groundwork for
these models, proposing methods for inference, optimization and
learning, and analyze their representational power.  We show that many
existing neural network architectures can be made input-convex with
a minor modification, and develop specialized optimization
algorithms tailored to this setting. Finally, we highlight the
performance of the methods on multi-label prediction, image
completion, and reinforcement learning problems, where we show
improvement over the existing state of the art in many cases.
\end{abstract}

\setcounter{footnote}{2} 
\section{Introduction}

In this paper, we propose a new
neural network architecture that we call the \emph{input convex neural network}
(ICNN).These are \emph{scalar-valued} neural networks $f(x, y;\theta)$ where $x$
and
$y$ denotes inputs to the function and $\theta$
denotes the parameters, built in such a way that the network is convex in
(a subset of) \emph{inputs} $y$.\footnotemark~
The fundamental benefit to these ICNNs is that we can \emph{optimize} over
the convex inputs to the network given some fixed value for other inputs.  That
is, given some fixed $x$ (and possibly some fixed elements of $y$) we can
globally and efficiently (because the problem is convex) solve the optimization
problem
\begin{equation}
\argmin_{y} f(x, y; \theta).
\label{eq:generic-argmin}
\end{equation}
Fundamentally, this
formalism lets us perform inference in the network via \emph{optimization}.
That is, instead of making predictions in a neural network via a purely
feedforward process, we can make predictions by optimizing a scalar function
(which effectively plays the role of an energy function) over some
inputs to the function given others.  There are a number of potential use cases
for these networks.

\footnotetext{We emphasize the term ``input convex''
since convexity in machine learning typically refers to convexity (of the loss
minimization learning problem) in the \emph{parameters}, which is not the case
here.  Note that in our notation, $f$ needs only be a convex function in
$y$, and may still be non-convex in the remaining inputs $x$.  Training these
neural networks remains a nonconvex problem, and the
convexity is only being exploited at inference time.}

\paragraph{Structured prediction}  As is perhaps apparent from our
notation above, a key application of this work is in structured prediction.
Given (typically high-dimensional) structured input and output spaces $\mathcal
{X} \times \mathcal{Y}$, we can build a network over $(x,y)$ pairs
that encodes the energy function for this pair, following typical energy-based
learning formalisms \citep{lecun2006tutorial}.  Prediction involves finding the
$y \in \mathcal {Y}$ that
minimizes the energy for a given $x$, which is exactly
the argmin problem in \eqref{eq:generic-argmin}.
In our setting, assuming that $\mathcal{Y}$ is a convex space (a common
assumption in structured prediction), this optimization
problem is convex.
This is similar in nature to the structured prediction energy networks (SPENs)
\citep{belanger2016structured}, which also use deep networks over the input and
output spaces, with the difference being that in our setting $f$ is convex in
$y$, so the optimization can be performed globally.

\paragraph{Data imputation}  Similar to structured prediction
but slightly more generic, if we are given some space $\mathcal{Y}$
we can learn a network $f(y;\theta)$ (removing the additional $x$
inputs, though these can be added as well) that, given an example with
some subset $\mathcal{I}$ missing, imputes the likely values of these variables
by solving the optimization problem as above $\hat{y}_\mathcal{I} = \argmin_
{y_\mathcal{I}} f(y_{\mathcal{I}}, y_{\bar {\mathcal{I}}}; \theta)$
This could be used e.g., in image inpainting
where the goal is to fill in some arbitrary set of missing pixels given observed
ones.

\paragraph{Continuous action reinforcement learning}
Given a reinforcement learning problem with potentially continuous
state and action spaces $\mathcal
{S} \times \mathcal{A}$,
we can model the (negative) $Q$ function,
$-Q(s,a;\theta)$ as an input convex neural network.  In this case the action
selection procedure can be formulated as a convex optimization problem
$a^\star(s) = \argmin_a -Q(s,a;\theta)$.

This paper lays the
foundation for optimization, inference, and learning in these input convex
models, and explores their performance in the applications above.  Our main
contributions are: we propose the ICNN
architecture and a partially convex variant; we develop
efficient optimization and inference procedures that are well-suited to the
complexity of these specific models; we propose techniques for training these
models, based upon either max-margin structured prediction or direct
differentiation of the argmin operation; and we evaluate the system on
multi-label prediction, image completion, and reinforcement learning domains;
in many of these settings we show performance that improves upon the
 state of the art.


\section{Background and related work}

\paragraph{Energy-based learning} The interplay between inference, optimization,
and structured prediction has a long history in neural networks.  Several early
incarnations of
neural networks were explicitly trained to produce structured sequences
(e.g. \cite{simard1991reverse}),
and there was an early appreciation that structured models like hidden Markov
models could be combined with the outputs of neural networks
\citep{bengio-lecun-henderson-94}.  Much of this earlier work is surveyed and
synthesized by \cite{lecun2006tutorial}, who give a tutorial on these energy
based learning
methods.
In recent years, there has been a strong push to further incorporate
structured prediction methods like conditional random fields as the ``last
layer'' of a deep network architecture
\citep{peng2009conditional,zheng2015conditional,chen2015learning}.
Several methods have proposed to build general neural networks over joint input
and output spaces, and perform inference over outputs using generic optimization
techniques such as Generative Adversarial Networks (GANs) \citep{goodfellow2014generative}
and Structured Prediction Energy Networks (SPENs) \citep{belanger2016structured}.
SPENs provide a deep structure over input and output spaces
that performs the inference in \eqref{eq:generic-argmin} as
a non-convex optimization problem.

The current work is highly related to these past approaches, but also differs in
a very particular way.  To the best of our knowledge, each of these structured
prediction methods based upon energy-based models operates in one of two ways,
either: 1) the architecture is built in a very particular way such that
optimization over the output is guaranteed to be ``easy'' (e.g. convex, or the
result of running some inference procedure), usually by introducing a
structured linear objective at the last layer of the network; or 2) no attempt
is made to make the architecture ``easy'' to run inference over, and instead a
general model is built over the output space.  In contrast, our approach lies
somewhere in between: by ensuring convexity of the resulting decision space, we
are constraining the inference problem to be easy in some respect, but we
specify very
little about the architecture other than the constraints required to make
it convex.  In particular, as we will show, the network architecture over the
variables to be optimized over can be deep and involve multiple
non-linearities.  The goal of the proposed work is to allow
for complex functions over the output without needing to specify them manually
(exactly analogous to how current deep neural networks treat their input space).

\paragraph{Structured prediction and MAP inference}
Our work also draws some
connection to MAP-inference-based learning and
approximate inference.  There are two broad classes of learning approaches in
structured prediction: method that use probabilistic inference techniques
(typically exploiting the fact that the gradient of log likelihood is given by
the actual feature expectations minus their expectation under the learned
model \citep[Ch 20]{koller2009probabilistic}), and methods that rely solely upon
MAP inference (such as max-margin structured prediction
\citep{taskar2005learning,tsochantaridis2005large}).  MAP inference in particular also
has close connections to optimization, as various convex relaxations of the
general MAP inference problem often perform well in theory and practice.
The proposed methods can be viewed as an extreme case of this second class of
algorithm, where inference is based \emph{solely} upon a convex optimization
problem that may not have any probabilistic semantics at all.  Finally, although
it is more abstract, we feel there is a philosophical similarity between our
proposed approach and sum-product networks \citep{poon2011sum}; both settings
define networks where inference is accomplished ``easily'' either by a
sum-product message passing algorithm (by construction) or via convex
optimization.

\paragraph{Fitting convex functions}
Finally, the proposed work relates to a
topic less considered in the machine learning literature, that of fitting convex
functions to data \citep[pg. 338]{boyd2004convex}.  Indeed our learning problem
can be viewed as
parameter estimation under a model that is guaranteed to be convex by its
construction.  The most similar work of which we
are aware specifically fits sums of rectified half-planes to data
\citep{magnani2009convex}, which
is similar to one layer of our rectified linear units.  However, the actual
training scheme is much different, and our deep network architecture allows for
a much richer class of representations, while still maintaining convexity.


\section{Convex neural network architectures}

Here we more formally present different ICNN architectures and
prove their convexity properties given certain constraints on the parameter
space.  Our chief claim is that the class of (full and partial) input convex
models is rich and lets us capture complex joint models over the input to a
network.

\subsection{Fully input convex neural networks}

\begin{figure}[t]
  \centering
  \includegraphics[width=0.4\textwidth]{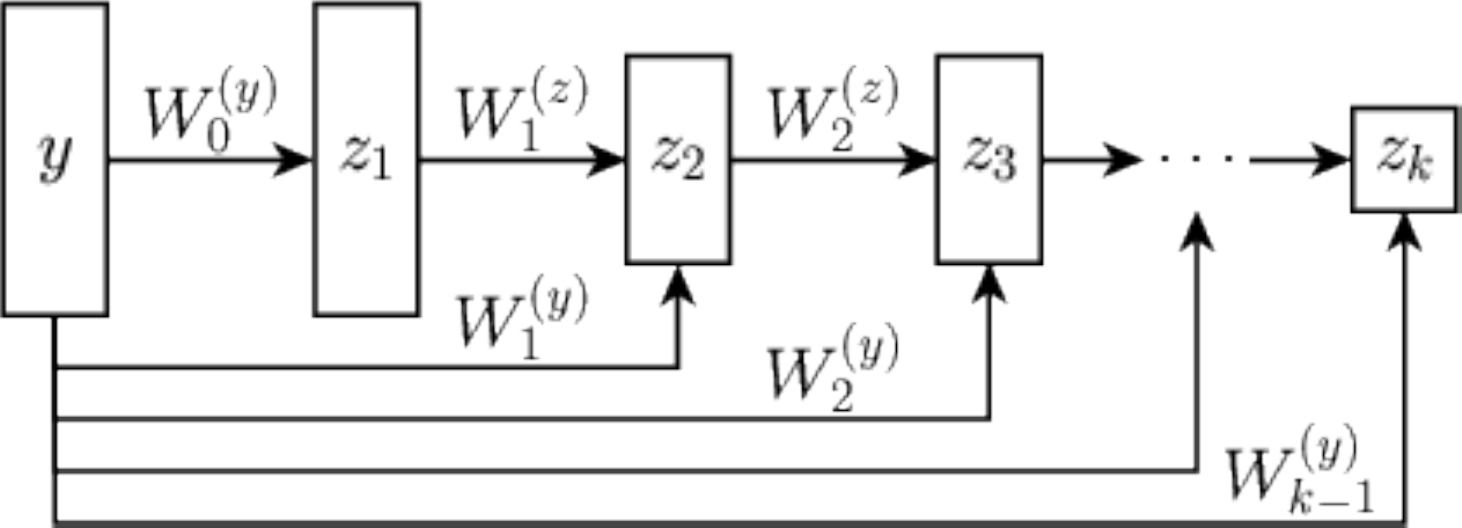}
  \caption{A fully input convex neural network (FICNN).}
  \label{fig:ficnn}
\end{figure}

To begin, we consider a fully convex, $k$-layer, fully connected ICNN
that we call a FICNN and is shown in Figure \ref{fig:ficnn}.
This model defines a neural network over the input $y$
(i.e., omitting any $x$ term in this function) using the architecture for
$i=0,\ldots,k-1$
\begin{equation}
\begin{split}
\label{eq-ficnn}
z_{i+1} & = g_i\left(W^{(z)}_i z_i + W^{(y)}_i y + b_i \right ), \;\; f
(y;\theta) = z_k
\end{split}
\end{equation}
where $z_i$ denotes the layer activations (with $z_0, W^{(z)}_0 \equiv 0$),
$\theta = \{W^ {(y)}_{0:k-1}, W^{(z)}_{1:k-1}, b_{0:k-1}\}$ are the
parameters, and $g_i$ are non-linear activation functions.  The central result
on convexity of the network is the following:
\begin{proposition}\label{prop-convex}
The function $f$ is convex in $y$ provided that all
$W^{(z)}_{1:k-1}$ are non-negative, and all functions $g_i$ are convex and
non-decreasing.
\end{proposition}
The proof is simple and follows from the fact that
non-negative sums of convex functions are also convex and that the composition
of a convex and convex non-decreasing function is also convex (see e.g.
\citet[3.2.4]{boyd2004convex}).  The constraint that the $g_i$ be convex
non-decreasing is not particularly restrictive, as current non-linear activation
units like the rectified linear unit or max-pooling unit already satisfy this
constraint.  The constraint that the $W^{(z)}$ terms be non-negative is
somewhat restrictive, but because the bias terms and $W^{(y)}$ terms can be
negative, the network still has substantial representation power, as we will
shortly demonstrate empirically.

One notable addition in the ICNN are the ``passthrough'' layers that directly
connect the input $y$ to hidden units in deeper layers.  Such layers are
unnecessary in traditional feedforward networks because previous hidden units
can always be mapped to subsequent hidden units with the identity mapping;
however, for ICNNs, the non-negativity constraint subsequent $W^{(z)}$ weights
restricts the allowable use of hidden units that mirror the identity mapping,
and so we explicitly include this additional passthrough.  Some
passthrough layers have been recently explored in the deep residual networks
\citep{he2015deep} and densely connected convolutional
networks \citep{huang2016densely},
though these differ from those of an ICNN as they pass through
hidden layers deeper in the network, whereas to maintain convexity our
passthrough layers can only apply to the input directly.

Other linear operators like convolutions can
be included in ICNNs without changing the convexity properties.
Indeed, modern feedforward architectures such as AlexNet
\citep{krizhevsky2012imagenet}, VGG \citep{simonyan2014very}, and GoogLeNet
\citep{szegedy2015going} with ReLUs \citep{nair2010rectified} can be made
input convex with Proposition~\ref{prop-convex}.  In the experiment that follow,
we will explore ICNNs with both fully connected and convolutional layers, and we
provide more detail about these additional architectures in
Section~\ref{sec:additional-arch} of the supplement.

\subsection{Partially input convex architectures}\label{sec:picnn}
\begin{figure}[t]
  \centering
  \includegraphics[width=0.4\textwidth]{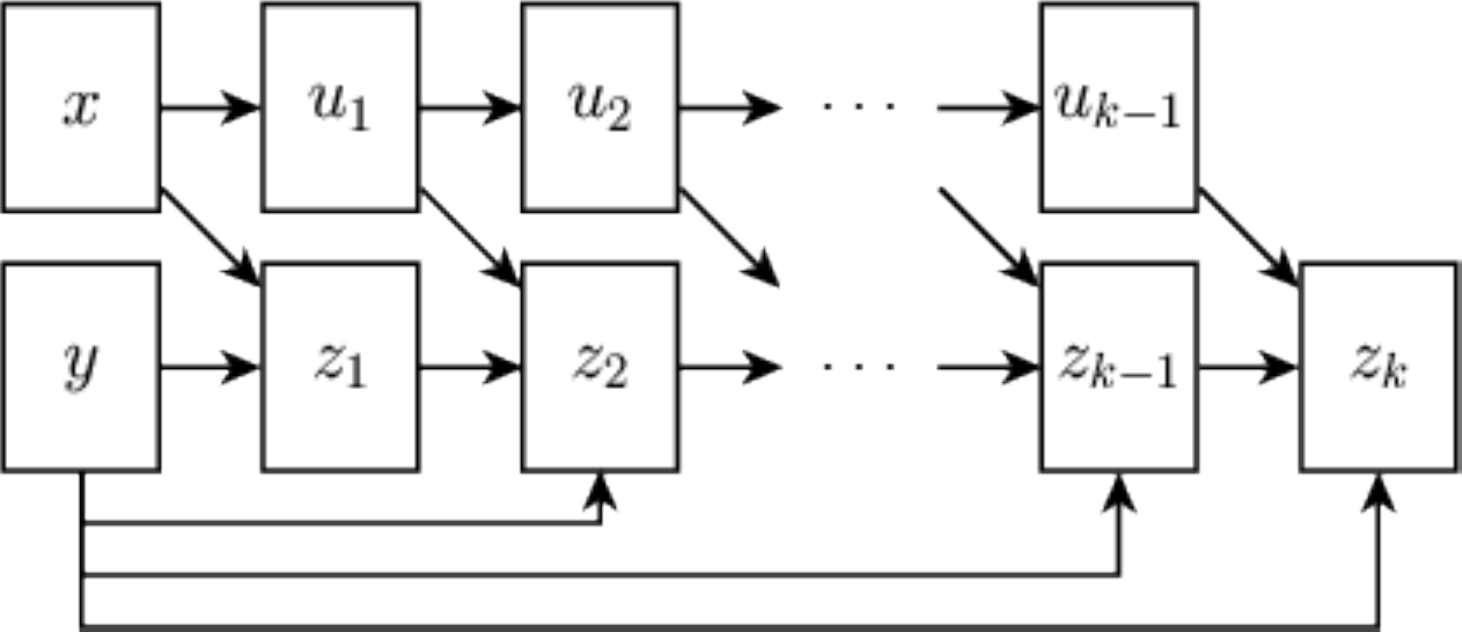}
  \caption{A partially input convex neural network (PICNN).}
  \label{fig:picnn}
\end{figure}
The FICNN provides joint convexity over the entire input to the function, which
indeed may
be a restriction on the allowable class of models.  Furthermore, this full joint
convexity is unnecessary in settings like structured prediction where the neural
network is used to build a joint model over an input and output example space
and only convexity over the outputs is necessary.

In this section we propose an extension to the pure FICNN, the partially
input convex neural network (PICNN), that is convex over only some inputs to the
network (in general ICNNs will refer to this new class). As we will show, these
networks generalize both
traditional feedforward networks and FICNNs, and thus provide substantial
representational benefits.  We define a PICNN to be a network over $(x,y)$ pairs
$f(x,y;\theta)$ where $f$ is convex in $y$ but not convex in $x$.
Figure~\ref{fig:picnn} illustrates one potential $k$-layer PICNN architecture
defined by the recurrences
\begin{equation}
\begin{split}
u_{i+1} & = \tilde{g}_i(\tilde{W}_i u_i + \tilde{b}_i) \\
z_{i+1} & = g_i \left( W^{(z)}_i \left (z_i \circ [W_i^{(zu)} u_i + b^{
      (z)}_i]_+ \right ) + \right. \\
  & \left. W^{(y)}_i \left (y \circ (W_i^{(yu)} u_i + b^{(y)}_i)\right)  + W^{(u)}_i
u_i + b_i \right ) \\
f(x,y;\theta) & = z_k, \; u_0 = x
\end{split}
\end{equation}
where $u_i \in \mathbb{R}^{n_i}$ and $z_i \in \mathbb{R}^{m_i}$ denote
the hidden units for the ``$x$-path'' and ``$y$-path'', where $y \in
\mathbb{R}^p$, and where $\circ$ denotes the Hadamard product, the
elementwise product between two vectors.  The crucial element here is that
unlike the FICNN, we only need the $W^{(z)}$ terms to be non-negative, and we
can introduce arbitrary products \emph{between} the $u_i$ hidden units and the
$z_i$ hidden units.
The following proposition highlights the representational
power of the PICNN.
\begin{proposition}
A PICNN network with $k$ layers can represent any FICNN with $k$ layers and any
purely feedforward network with $k$ layers.
\end{proposition}
\begin{proof}
To recover a FICNN we simply set the weights over the entire $x$ path to be
zero and set $b^{(z)} = b^{(y)} = 1$.  We can recover a feedforward network by
noting that a traditional feedforward network $\hat{f}(x;\theta)$ where $f :
\mathcal{X} \rightarrow \mathcal{Y}$,
can be viewed as a network with an inner
product $f(x;\theta)^T y$ in its last layer
(see e.g. \cite{lecun2006tutorial} for more details).
Thus, a feedforward network can be represented as a PICNN
by setting the $x$ path to be exactly the feedforward component, then having the
$y$ path be all zero except $W_{k-1}^{(yu)} = I$ and $W^{(y)}_{k-1} = 1^T$.
\end{proof}

\section{Inference in ICNNs}

Prediction in ICNNs (which we also refer to as inference), requires
solving the convex optimization problem
\begin{equation}
\label{eq-min-2}
\minimize_{y \in \mathcal{Y}} f(x,y;\theta)
\end{equation}
While the resulting tasks are convex optimization problems (and thus
``easy'' to solve in some sense), in practice this still involves the solution
of a potentially very complex optimization problem.
We discuss here several approaches for approximately solving these optimization
problems.  We can usually obtain reasonably accurate solutions
in many settings using a procedure that only involves a small number of forward
and backward passes through the network, and which thus has a complexity that
is at most a constant factor worse than that for feedforward networks.
The same consideration will apply to training such networks, which we will
discuss in Section \ref{sec:learning}.

\paragraph{Exact inference in ICNNs}Although it is not a practical approach for
solving the optimization tasks,
the inference problem for the networks presented above (where the non-linear
are either ReLU or linear units) can be posed as as linear program.
We show how to do this in Section~\ref{sec:exact-inference}.

\subsection{Approximate inference in ICNNs}
Because of the impracticality of exact inference, we focus on approximate
approaches to optimizing over the inputs to these networks, but ideally ones
that still exploit the convexity of the resulting problem.  We
specifically focus on gradient-based approaches, which use the fact that we can
easily compute the gradient of an ICNN with respect to its inputs, $\nabla_y f
(x,y;\theta)$, using backpropagation.

\textbf{Gradient descent}. The simplest gradient-based methods for solving
\eqref{eq-min-2} is just (projected sub-) gradient descent, or modifications
such as
those that use a momentum term \citep{polyak1964some,rumelhart1988learning}, or
spectral step size modifications \citep{barzilai1988two,birgin2000nonmonotone}.
That is, we start with some initial $\hat{y}$ and repeat the update
\begin{equation}
\hat{y} \leftarrow \mathcal{P}_\mathcal{Y} \left (\hat{y} - \alpha \nabla_y f
(x,\hat{y};\theta) \right )
\end{equation}
This method is appealing in its simplicity, but suffers from the typical
problems of gradient descent on non-smooth objectives: we need to pick a
step size and possibly use a sequence of decreasing step sizes, and don't have an
obvious method to assess how accurate of a current solution we have obtained
(since an ICNN with ReLUs is piecewise linear, it will not have
zero gradient at the solution).
The method is also more challenging to integrate with some learning
procedures, as we often need to differentiate through an entire chain of the
gradient descent algorithm \cite{domke2012generic}. Thus, while the method can
sometimes
work in practice, we have found that other approaches typically far outperform
this method, and we will focus on alternative approximate approaches for the
remainder of this section.

\subsection{Approximate inference via the bundle entropy method}
An alternative approach to gradient descent is the bundle method
\citep{smola2007bundle}, also known as the epigraph cutting plane approach,
which iteratively optimizes a piecewise lower bound on the function given by the
maximum over a set of first-order approximations.  However, as, the traditional
bundle method is not well suited to our setting (we need to evaluate a number
of gradients equal to the dimension of $x$, and solve a complex optimization
problem at each step) we have developed a new optimization algorithm for
this domain that we term the \emph{bundle entropy method}.  This algorithm
specifically applies to the (common) case where $\mathcal{Y}$
is bounded, which we assume to be $\mathcal{Y} = [0,1]^n$
(other upper or lower bounds can be attained through scaling).
The method is also easily extensible to the setting where elements
of $\mathcal{Y}$ belong to a higher-dimensional probability simplex as well.

For this approach, we consider adding an additional ``barrier'' function to the
optimization in the form of the negative entropy $-H(y)$, where
\begin{equation}
H(y) = -\sum_{i=1}^n(y_i\log y_i + (1-y_i)\log(1-y_i)).
\end{equation}
In other words, we instead want to solve the optimization problem $\argmin_y
\;\; f(x,y;\theta) - H(y)$ (with a possible additional scaling term).
The negative entropy is a convex function, with the limits
of $\lim_{y\rightarrow 0} H (y)=\lim_{y\rightarrow1} H(y) = 0$, and negative
values in the interior of this range.  The function acts as a barrier because,
although it does not approach infinity as it reaches the barrier of the feasible
set, its gradient \emph{does} approach infinity as it reaches the barrier, and
thus the optimal solution will always lie in the interior of the unit hypercube
$\mathcal{Y}$.

An appealing feature of the entropy regularization comes from its close
connection with sigmoid units in typical neural networks.  It follows easily
from first-order optimality conditions that the optimization problem
\begin{equation}
\minimize_y \; c^T y - H(y)
\end{equation}
is given by $y^\star = 1/(1+\exp(c))$.  Thus if we consider the ``trivial''
PICNN mentioned in Section \ref{sec:picnn}, which simply consists of the
function $f(x,y;\theta) = y^T \tilde{f}(x;\theta)$ for some purely feedforward
network $\tilde{f}(x;\theta)$, then the entropy-regularized minimization problem
gives a solution that is equivalent to simply taking the sigmoid of the neural
network outputs.  Thus, the move to ICNNs can be interpreted as providing a
more structured joint energy functional over the linear function implicitly used
by sigmoid layers.

At each iteration of the bundle entropy method, we solve the optimization
problem
\begin{equation}
y^{k+1}, t^{k+1} :=  \argmin_{y, t} \;\; \{t - H(y) \mid G y + h \leq t1 \}
\end{equation}
where $G \in \mathbb{R}^{k \times n}$ has rows equal to
\begin{equation}
g_i^T = \nabla_y f (x, y^i;\theta)^T
\end{equation}
and $h \in \mathbb{R}^k$ has entries equal to
\begin{equation}
h_i = f(x, y^i;\theta) - \nabla_y f (x, y^i;\theta)^T y^i.
\end{equation}
The Lagrangian of the optimization problem is
\begin{equation}
\mathcal{L}(y,t,\lambda) = t - H(y) + \lambda^T(G y + h - t1)
\end{equation}
and differentiating with respect to $y$ and $t$ gives the optimality conditions
\begin{equation}
\begin{split}
\nabla_y \mathcal{L}(y,t,\lambda) = 0 & \; \Longrightarrow \; y = \frac{1}{1 +
\exp(G^T \lambda)} \\
\nabla_t \mathcal{L}(y,t,\lambda) = 0 & \; \Longrightarrow \; 1^T \lambda = 1
\end{split}
\end{equation}
which in turn leads to the dual problem
\begin{equation}
\begin{split}
\maximize_\lambda \;\; &(G1 + h)^T \lambda - 1^T\log(1 + \exp(G^T\lambda)) \\
\subjectto \;\; & \lambda \geq 0, 1^T \lambda = 1.
\end{split}
\end{equation}
This is a smooth optimization problem over the unit simplex, and can be solved using
a method like
the Projected Newton method of \citep[pg. 241, eq. 97]{bertsekas1982projected}.
A complete description of the bundle entropy
method is given in Section~\ref{sec:bundle-entropy-alg}.
For lower dimensional problems, the bundle entropy method often attains an exact
solution after a relatively small number of iterations. And even for larger
problems, we find that the approximate solutions generated by a very
small number of iterations (we typically use 5 iterations), still
substantially outperform gradient descent approaches. Further, because we
maintain an explicit lower bound on the function, we can compute an
optimality gap of our solution, though in practice just using a fixed number of
iterations performs well.


\section{Learning ICNNs}
\label{sec:learning}

Generally speaking, ICNN learning shapes the objective's energy function to
produce the desired values when optimizing over the relevant inputs.  That is,
for a given input output pair $(x,y^\star)$, our goal is to find ICNN parameters
$\theta$ such that
\begin{equation}
y^\star \approx \argmin_y \tilde{f}(x,y;\theta)
\end{equation}
where for the entirely of this section, we use the notation $\tilde{f}$ to
denote the combination of the neural network function \emph{plus} the
regularization term such as $-H(y)$, if it is included, i.e.
\begin{equation}
\tilde{f}(x,y;\theta) = f(x,y;\theta) - H(y).
\end{equation}
Although we only discuss the entropy regularization in this work, we emphasize
that other regularizers are also possible. Depending on the setting, there are
several different approaches we can use
to ensure that the ICNN achieves the desired targets, and we consider three
approaches below: direct functional fitting,
max-margin structured prediction, and argmin differentiation.

\paragraph{Direct functional fitting.}
We first note that in some domains, we do not need a specialized procedure for
fitting ICNNs, but can use existing approaches that directly fit the ICNN.
An example of this is the Q-learning setting.
Given some observed tuple $(s,a,r,s')$, Q learning updates the
parameters $\theta$ with the gradient
\begin{equation}
\label{eq:q-learning-update}
\left(Q(s,a) - r - \gamma\max_{a'} Q (s',a') \right )\nabla_\theta Q(s,a),
\end{equation}
where the maximization step is carried out with gradient descent or
the bundle entropy method.
These updates can be applied to ICNNs with the only additional requirement that
we project the weights onto their feasible sets after this update
(i.e., clip or project any $W$ terms that are required to be positive).  Section~\ref{sec:icnn-rl} gives a complete description of
deep Q-learning with ICNNs.


\paragraph{Max-margin structured prediction.}
Although max-margin structured prediction is a simple and well-studied approach
\citep{tsochantaridis2005large,taskar2005learning},
in our experiences using these methods within an ICNN, we had substantial
difficulty choosing the proper margin scaling term (especially for domains with
continuous-valued outputs), or allowing for losses other
than the hinge loss.  For this reason,
Section~\ref{sec:max-margin} discusses max-margin structured
prediction in more detail, but the majority of our experiments here
focus on the next approach, which more directly encodes the
loss suffered by the full structured-prediction pipeline.

\subsection{Argmin differentiation}

In our final proposed approach, that of argmin differentiation, we
propose to directly minimize a loss function between true outputs and the
outputs predicted by our model, where these predictions themselves are the
result of an optimization problem.  We explicitly consider the case where the
approximate solution to the inference problem is attained via the
previously-described bundle entropy method, typically run for some fixed
(usually small) number of iterations. To simplify notation, in the following we
will let
\begin{equation}
  \begin{split}
\hat{y}(x;\theta) &= \argmin_{y} \min_t \;\; \{t - H(y) \mid G y + h \leq t1 \} \\
& \approx \argmin_y \tilde{f}(x,y;\theta)
  \end{split}
\end{equation}
refer to the \emph{approximate} minimization over $y$ that
results from running the bundle entropy method, specifically at the last
iteration of the method.

Given some example $(x,y^\star)$, our goal is to compute the gradient, with
respect to the ICNN parameters, of the loss between $y^\star$ and $\hat{y}
(x;\theta)$:
$\ell(\hat{y}(x;\theta), y^\star)$.  This is in some sense the most direct
analogue to
traditional neural network learning, since we typically optimize networks by
minimizing some loss between the network's (feedforward) predictions and the true
desired labels.  Doing this in the predictions-via-optimization setting
requires that we differentiate ``through'' the argmin operator, which can be
accomplished via implicit differentiation of the KKT optimality conditions.
Although the derivation is somewhat involved, the final result is fairly
compact, and is given by the following proposition (for simplicity, we will
write $\hat{y}$ below instead of $\hat{y}(x;\theta)$ when the notation should
be clear):
\begin{proposition}
\label{proposition-gradient}
The gradient of the neural network loss for predictions generated through
the minimization process is
\begin{equation}
  \begin{split}
    & \nabla_\theta \ell(\hat{y}(x;\theta), y^\star) = \sum_{i=1}^k (c^\lambda_i
    \nabla_\theta f(x, y^i;\theta) + \\
    & \hspace{2mm} \nabla_\theta \left(\nabla_y f(x,y^i;\theta)^T
\left (\lambda_i c^y + c^\lambda_i \left (\hat{y}(x;\theta) - y^i \right )
\right
)\right) )
  \end{split}
\end{equation}
where $y^i$ denotes the solution returned by the $i$th iteration of the entropy
bundle method, $\lambda$ denotes the dual variable solution of the entropy
bundle method, and where the $c$ variables are determined by the solution to the
linear system
\begin{equation}
\left [ \begin{array}{ccc}
H & G^T & 0 \\
G & 0 & -1 \\
0 & -1^T & 0 \end{array} \right ]
\left [ \begin{array}{c} c^y \\ c^\lambda \\ c^t \end{array} \right ]
=
\left [ \begin{array}{c} -\nabla_{\hat{y}} \ell(\hat{y}, y^\star) \\ 0 \\ 0
\end{array} \right ].
\end{equation}
where $H = \diag\left(\frac{1}{\hat{y}} + \frac{1}{1-\hat{y}}\right)$.
\end{proposition}

The proof of this proposition is given in Section~\ref{sec:argmin-diff-proof},
but we highlight a few key points here.
The complexity of computing this gradient will be linear in $k$, which is the
number of \emph{active} constraints at the solution of the bundle entropy
method.  The inverse of this matrix can also be computed efficiently by just
inverting the $k
\times k$ matrix $G H^{-1} G^T$ via
a variable elimination procedure, instead of by inverting the full matrix.
The gradients $\nabla_\theta f(x,y_i;\theta)$ are standard neural network
gradients, and further, can be computed in the same forward/backward pass as we
use to compute the gradients for the bundle entropy method.
The main challenge of the method is to compute the terms of the form
$\nabla_\theta (\nabla_y f(x,y_i;\theta)^T v)$
for some vector $v$.  This quantity can be computed by most autodifferentiation
tools (the gradient inner product $\nabla_y f(x,y_i;\theta)^T v$ itself just
becomes a graph computation than can be differentiated itself), or it can be
computed by a finite difference approximation.
The complexity of computing this entire gradient is a small
constant multiple of computing $k$ gradients with respect to $\theta$.

Given this ability to compute gradients with respect to an arbitrary loss
function, we can fit the parameter using traditional stochastic gradient methods
examples.  Specifically, given an example (or a minibatch of examples) $x_i,
y_i$, we compute gradients $\nabla_\theta \ell(\hat{y}(x_i;\theta), y_i)$
and update the parameters using e.g. the ADAM optimizer \citep{kingma2014adam}.


\section{Experiments}

Our experiments study the representational power of ICNNs to
better understand the interplay between the model's
restrictiveness and accuracy.  Specifically, we evaluate the method on
multi-label classification on the BibTeX dataset \citep{katakis2008multilabel},
image completion using the Olivetti face dataset \citep{samaria1994parameterisation},
and continuous action reinforcement learning in the
OpenAI Gym \citep{brockman2016openai}. We show that the methods compare
favorably to the state of the art in many situations.
The full source code for all experiments is available in the
\verb!icml2017! branch at \url{https://github.com/locuslab/icnn} and
our implementation is built using
Python \citep{van1995python} with the numpy \citep{oliphant2006guide} and
TensorFlow \citep{abadi2016tensorflow} packages.


\subsection{Synthetic 2D example}  Though we do not discuss it here,
Section~\ref{sec:synthetic} presents a simple synthetic classification
experiment comparing FICNN and PICNN decision boundaries.

\subsection{Multi-Label Classification}
We first study how ICNNs perform on multi-label classification with the
BibTeX dataset and benchmark presented in \cite{katakis2008multilabel}.
This benchmark maps text classification from an input space $\cal{X}$ of
1836 bag-of-works indicator (binary) features to an output
space $\cal{Y}$ of 159 binary labels.
We use the train/test split of 4880/2515 from \citep{katakis2008multilabel}
and evaluate with the macro-F1 score (higher is better).
We use the ARFF version of this dataset from Mulan \citep{tsoumakas2011mulan}.
Our PICNN architecture for multi-label classification uses fully-connected
layers with ReLU activation functions and batch
normalization \citep{ioffe2015batch} along the input path.
As a baseline, we use a fully-connected neural network with
batch normalization and ReLU activation functions.
Both architectures have the same structure
(600 fully connected, 159 (\#labels) fully connected).
We optimize our PICNN with 30 iterations of gradient descent
with a learning rate of 0.1 and a momentum of 0.3.

\begin{table}
\begin{center}
\begin{tabular}{@{}ll@{}}
Method & Test Macro-F1 \\ \hline
Feedforward net & 0.396 \\
ICNN & 0.415 \\
SPEN \citep{belanger2016structured} & \textbf{0.422} \\
\end{tabular}
\caption{Comparison of approaches on BibTeX multi-label classification task.
(Higher is better.)}
\label{table:bibtex}
\end{center}
\end{table}
Table \ref{table:bibtex} compares several different methods for this problem.
Our PICNN's final macro-F1 score of 0.415 outperforms our
baseline feedforward network's score of 0.396,
which indicates PICNNs have the power to learn a robust
structure over the output space.
SPENs obtain a macro-F1 score of 0.422 on this task \citep{belanger2016structured}
and pose an interesting comparison point to ICNNs as they have
a similar (but not identical) deep structure that is non-convex
over the input space.
The difference of 0.007 between ICNNs and SPENs could be due
to differences in our experimental setups, architectures,
and random experimental noise.
More details are included in Section~\ref{sec:exp:ml:f1}.

\subsection{Image completion on the Olivetti faces}
As a test of the system on a structured prediction task over a
much more complex output space $\cal{Y}$, we apply a
convolutional PICNN to face completion on the
sklearn version \citep{pedregosa2011scikit} of the Olivetti
data set \citep{samaria1994parameterisation}, which contains
400 64x64 grayscale images.
ICNNs for face completion should be invariant to translations
and other transformations in the input space.
To achieve this invariance, our PICNN is inspired by the
DQN architecture in \citet{mnih2015human}, which preserves
this invariance in the different context of reinforcement learning.
Specifically, our network
is over $(x,y)$ pairs where
$x$ (32x64) is the left half and $y$ (32x64)
is the right half of the image.
The input and output paths are:
32x8x8 conv (stride 4x2), 64x4x4 conv (stride 2x2),
64x3x3 conv, 512 fully connected.

This experiment uses the same training/test splits and minimizes
the mean squared error (MSE) as in \citet{poon2011sum} so that our
results can be directly compared to (a non-exhaustive list of)
other techniques.
We also explore the tradeoffs between the bundle entropy method
and gradient descent and use a non-convex baseline to
better understand the impacts of convexity.
We use a learning rate of 0.01 and momentum of 0.9 with
gradient descent for the inner optimization in the ICNN.

\begin{figure}[t]
  \centering
  \includegraphics[width=0.35\textwidth]{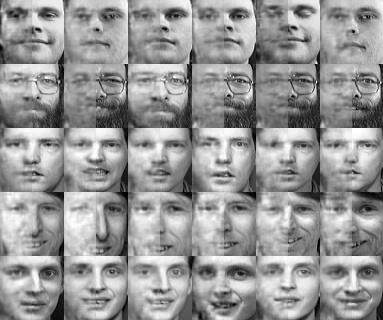}
  \caption{Example test set image completions of the ICNN with bundle entropy.}
  \label{fig:images-completed}
\end{figure}

Table~\ref{table:image} shows the test MSEs for the different approaches.
Example image completions are shown in Figure~\ref{fig:images-completed}.
These results show that the bundle entropy method can leverage
more information from these five iterations than gradient descent,
even when the convexity constraint is relaxed.
The PICNN trained with back-optimization with the relaxed convexity
constraint slightly outperforms the network with
the convexity constraint, but not the network trained with the
bundle-entropy method.
This shows that for image completion with PICNNs, convexity does not seem to
inhibit the representational power.
Furthermore, this experiment suggests that a small number of inner optimization iterations
(five in this case) is sufficient for good performance.
\begin{table}
\begin{center}
\begin{tabular}{@{}ll@{}}
Method & MSE \\ \hline
ICNN - Bundle Entropy & \textbf{833.0} \\
ICNN - Gradient Decent & 872.0 \\
ICNN - Nonconvex & 850.9 \\
Sum-product \citep{poon2011sum} & 942
\end{tabular}
\caption{Comparisons of reconstruction error on image completion.}
\label{table:image}
\end{center}
\end{table}

\subsection{Continuous Action Reinforcement Learning}
Finally, we present standard benchmarks in continuous action reinforcement
learning from the OpenAI Gym \citep{brockman2016openai} that use the
MuJoCo physics simulator \citep{todorov2012mujoco}.
We model the (negative) $Q$ function,
$-Q(s,a;\theta)$ as an ICNN and select actions with
the convex optimization problem
$a^\star(s) = \argmin_a -Q(s,a;\theta)$.
We use Q-learning to optimize the ICNN as described in
Section~\ref{sec:learning} and Section~\ref{sec:icnn-rl}.
At test time, the policy is selected by optimizing $Q(s, a; \theta)$.
All of our experiments use a PICNN with two fully-connected
layers that each have 200 hidden units.
We compare to Deep Deterministic Policy Gradient (DDPG) \citep{lillicrap2015continuous}
and Normalized Advantage Functions (NAF) \citep{gu2016continuous}
as state-of-the-art off-policy learning baselines.\footnote{Because there are
not official DDPG or NAF implementations or
results on the OpenAI gym tasks, we use the Simon Ramstedt's
DDPG implementation from \url{https://github.com/SimonRamstedt/ddpg}
and have re-implemented NAF.}

\begin{table}
\begin{center}
\begin{tabular}[c]{@{}llll@{}}
Task & DDPG & NAF & ICNN \\ \hline
Ant & 1000.00 & 999.03 & \textbf{1056.29} \\
HalfCheetah & 2909.77 & 2575.16 & \textbf{3822.99} \\
Hopper & \textbf{1501.33} & 1100.43 & 831.00 \\
Humanoid & 524.09 & \textbf{5000.68} & 433.38 \\
HumanoidStandup & 134265.96 & 116399.05 & \textbf{141217.38} \\
InvDoubPend & \textbf{9358.81} & \textbf{9359.59} & \textbf{9359.41} \\
InvPend & \textbf{1000.00} & \textbf{1000.00} & \textbf{1000.00} \\
Reacher & -6.10 & -6.31 & \textbf{-5.08} \\
Swimmer & 49.79 & \textbf{69.71} & 64.89 \\
Walker2d & \textbf{1604.18} & 1007.25 & 298.21 \\
\end{tabular}

\caption{Maximum test reward for ICNN algorithm versus alternatives on several
OpenAI Gym tasks. (All tasks are v1.)}
\label{tab:rl:maxTestRew}
\end{center}
\end{table}

Table~\ref{tab:rl:maxTestRew} shows the maximum test reward achieved by the
different algorithms on these tasks.  Although no method strictly dominates the
others, the ICNN approach has some clear advantages on tasks like HalfCheetah,
Reacher, and HumanoidStandup, and performs comparably on many other tasks,
though with also a few notable poor performances in Hopper and Walker2D.
Nonetheless, given the strong baseline, and the fact that the method is
literally just a drop-in replacement for a function approximator in Q-learning,
these results are overall positive.
NAF poses a particularly interesting comparison point to ICNNs.  In particular,
NAF decomposes the $Q$ function in terms of the value function an an advantage
function $Q(s,a) = V(s) + A(s,a)$ where the advantage function is restricted to
be \emph{concave quadratic} in the actions, and thus always has a closed-form
solution.  In a sense, this closely mirrors the setup of the PICNN architecture:
like NAF, we have a separate non-convex path for the $s$ variables, and an
overall function that is convex in $a$; however, the distinction is that while
NAF requires that the convex portion be quadratic, the ICNN architecture allows
any convex functional form.  As our experiments show, this representational
power does allow for better performance of the resulting system, though the
trade-off, of course, is that determining the optimal action in an ICNN is
substantially more computationally complex than for a quadratic.

\section{Conclusion and future work}

This paper laid the groundwork for the input convex neural network model.  By
incorporating relatively simple constraints into existing network architectures,
we can fit very general convex functions and the apply optimization as an
inference procedure.  Since many existing models already fit into this overall
framework (e.g., CRF models perform an optimization over an output space where
parameters are given by the output of a neural network), the proposed method
presents an extension where the entire inference procedure is ``learned'' along
with the network itself, without the need for explicitly building typical
structured prediction architectures.  This work explored only a small subset of
the possible applications of these network, and the networks offer promising
directions for many additional domains.

\newpage
\section*{Acknowledgments}
BA is supported by the National Science Foundation Graduate Research Fellowship
Program under Grant No. DGE1252522.
We also thank David Belanger for helpful discussions.

\bibliography{icnn}

\begin{thebibliography}{45}
\providecommand{\natexlab}[1]{#1}
\providecommand{\url}[1]{\texttt{#1}}
\expandafter\ifx\csname urlstyle\endcsname\relax
  \providecommand{\doi}[1]{doi: #1}\else
  \providecommand{\doi}{doi: \begingroup \urlstyle{rm}\Url}\fi

\bibitem[Abadi et~al.(2016)Abadi, Agarwal, Barham, Brevdo, Chen, Citro,
  Corrado, Davis, Dean, Devin, et~al.]{abadi2016tensorflow}
Abadi, Mart{\i}n, Agarwal, Ashish, Barham, Paul, Brevdo, Eugene, Chen, Zhifeng,
  Citro, Craig, Corrado, Greg~S, Davis, Andy, Dean, Jeffrey, Devin, Matthieu,
  et~al.
\newblock Tensorflow: Large-scale machine learning on heterogeneous distributed
  systems.
\newblock \emph{arXiv preprint arXiv:1603.04467}, 2016.

\bibitem[Barzilai \& Borwein(1988)Barzilai and Borwein]{barzilai1988two}
Barzilai, Jonathan and Borwein, Jonathan~M.
\newblock Two-point step size gradient methods.
\newblock \emph{IMA Journal of Numerical Analysis}, 8\penalty0 (1):\penalty0
  141--148, 1988.

\bibitem[Belanger \& McCallum(2016)Belanger and
  McCallum]{belanger2016structured}
Belanger, David and McCallum, Andrew.
\newblock Structured prediction energy networks.
\newblock In \emph{Proceedings of the International Conference on Machine
  Learning}, 2016.

\bibitem[Bengio et~al.(1994)Bengio, LeCun, and
  Henderson]{bengio-lecun-henderson-94}
Bengio, Yoshua, LeCun, Yann, and Henderson, Donnie.
\newblock Globally trained handwritten word recognizer using spatial
  representation, convolutional neural networks, and hidden markov models.
\newblock \emph{Advances in neural information processing systems}, pp.\
  937--937, 1994.

\bibitem[Bertsekas(1982)]{bertsekas1982projected}
Bertsekas, Dimitri~P.
\newblock Projected newton methods for optimization problems with simple
  constraints.
\newblock \emph{SIAM Journal on control and Optimization}, 20\penalty0
  (2):\penalty0 221--246, 1982.

\bibitem[Birgin et~al.(2000)Birgin, Mart{\'\i}nez, and
  Raydan]{birgin2000nonmonotone}
Birgin, Ernesto~G, Mart{\'\i}nez, Jos{\'e}~Mario, and Raydan, Marcos.
\newblock Nonmonotone spectral projected gradient methods on convex sets.
\newblock \emph{SIAM Journal on Optimization}, 10\penalty0 (4):\penalty0
  1196--1211, 2000.

\bibitem[Boyd \& Vandenberghe(2004)Boyd and Vandenberghe]{boyd2004convex}
Boyd, Stephen and Vandenberghe, Lieven.
\newblock \emph{Convex optimization}.
\newblock Cambridge university press, 2004.

\bibitem[Boyd et~al.(2011)Boyd, Parikh, Chu, Peleato, and
  Eckstein]{boyd2011distributed}
Boyd, Stephen, Parikh, Neal, Chu, Eric, Peleato, Borja, and Eckstein, Jonathan.
\newblock Distributed optimization and statistical learning via the alternating
  direction method of multipliers.
\newblock \emph{Foundations and Trends{\textregistered} in Machine Learning},
  3\penalty0 (1):\penalty0 1--122, 2011.

\bibitem[Brockman et~al.(2016)Brockman, Cheung, Pettersson, Schneider,
  Schulman, Tang, and Zaremba]{brockman2016openai}
Brockman, Greg, Cheung, Vicki, Pettersson, Ludwig, Schneider, Jonas, Schulman,
  John, Tang, Jie, and Zaremba, Wojciech.
\newblock Openai gym.
\newblock \emph{arXiv preprint arXiv:1606.01540}, 2016.

\bibitem[Chen et~al.(2015)Chen, Schwing, Yuille, and Urtasun]{chen2015learning}
Chen, Liang-Chieh, Schwing, Alexander~G, Yuille, Alan~L, and Urtasun, Raquel.
\newblock Learning deep structured models.
\newblock In \emph{Proceedings of the International Conference on Machine
  Learning}, 2015.

\bibitem[Domke(2012)]{domke2012generic}
Domke, Justin.
\newblock Generic methods for optimization-based modeling.
\newblock In \emph{Proceedings of the Conference on {AI} and Statistics}, pp.\
  318--326, 2012.

\bibitem[Duchi et~al.(2011)Duchi, Hazan, and Singer]{duchi2011adaptive}
Duchi, John, Hazan, Elad, and Singer, Yoram.
\newblock Adaptive subgradient methods for online learning and stochastic
  optimization.
\newblock \emph{The Journal of Machine Learning Research}, 12:\penalty0
  2121--2159, 2011.

\bibitem[Goodfellow et~al.(2014)Goodfellow, Pouget-Abadie, Mirza, Xu,
  Warde-Farley, Ozair, Courville, and Bengio]{goodfellow2014generative}
Goodfellow, Ian, Pouget-Abadie, Jean, Mirza, Mehdi, Xu, Bing, Warde-Farley,
  David, Ozair, Sherjil, Courville, Aaron, and Bengio, Yoshua.
\newblock Generative adversarial nets.
\newblock In \emph{Advances in Neural Information Processing Systems}, pp.\
  2672--2680, 2014.

\bibitem[Gu et~al.(2016)Gu, Lillicrap, Sutskever, and Levine]{gu2016continuous}
Gu, Shixiang, Lillicrap, Timothy, Sutskever, Ilya, and Levine, Sergey.
\newblock Continuous deep q-learning with model-based acceleration.
\newblock In \emph{Proceedings of the International Conference on Machine
  Learning}, 2016.

\bibitem[He et~al.(2015)He, Zhang, Ren, and Sun]{he2015deep}
He, Kaiming, Zhang, Xiangyu, Ren, Shaoqing, and Sun, Jian.
\newblock Deep residual learning for image recognition.
\newblock \emph{arXiv preprint arXiv:1512.03385}, 2015.

\bibitem[Huang et~al.(2016)Huang, Liu, and Weinberger]{huang2016densely}
Huang, Gao, Liu, Zhuang, and Weinberger, Kilian~Q.
\newblock Densely connected convolutional networks.
\newblock \emph{arXiv preprint arXiv:1608.06993}, 2016.

\bibitem[Ioffe \& Szegedy(2015)Ioffe and Szegedy]{ioffe2015batch}
Ioffe, Sergey and Szegedy, Christian.
\newblock Batch normalization: Accelerating deep network training by reducing
  internal covariate shift.
\newblock In \emph{Proceedings of The 32nd International Conference on Machine
  Learning}, pp.\  448--456, 2015.

\bibitem[Katakis et~al.(2008)Katakis, Tsoumakas, and
  Vlahavas]{katakis2008multilabel}
Katakis, Ioannis, Tsoumakas, Grigorios, and Vlahavas, Ioannis.
\newblock Multilabel text classification for automated tag suggestion.
\newblock \emph{ECML PKDD discovery challenge}, 75, 2008.

\bibitem[Kingma \& Ba(2014)Kingma and Ba]{kingma2014adam}
Kingma, Diederik and Ba, Jimmy.
\newblock Adam: A method for stochastic optimization.
\newblock \emph{arXiv preprint arXiv:1412.6980}, 2014.

\bibitem[Koller \& Friedman(2009)Koller and Friedman]{koller2009probabilistic}
Koller, Daphne and Friedman, Nir.
\newblock \emph{Probabilistic graphical models: principles and techniques}.
\newblock MIT press, 2009.

\bibitem[Krizhevsky et~al.(2012)Krizhevsky, Sutskever, and
  Hinton]{krizhevsky2012imagenet}
Krizhevsky, Alex, Sutskever, Ilya, and Hinton, Geoffrey~E.
\newblock Imagenet classification with deep convolutional neural networks.
\newblock In \emph{Advances in neural information processing systems}, pp.\
  1097--1105, 2012.

\bibitem[LeCun et~al.(2006)LeCun, Chopra, Hadsell, Ranzato, and
  Huang]{lecun2006tutorial}
LeCun, Yann, Chopra, Sumit, Hadsell, Raia, Ranzato, M, and Huang, F.
\newblock A tutorial on energy-based learning.
\newblock \emph{Predicting structured data}, 1:\penalty0 0, 2006.

\bibitem[Lillicrap et~al.(2015)Lillicrap, Hunt, Pritzel, Heess, Erez, Tassa,
  Silver, and Wierstra]{lillicrap2015continuous}
Lillicrap, Timothy~P, Hunt, Jonathan~J, Pritzel, Alexander, Heess, Nicolas,
  Erez, Tom, Tassa, Yuval, Silver, David, and Wierstra, Daan.
\newblock Continuous control with deep reinforcement learning.
\newblock \emph{arXiv preprint arXiv:1509.02971}, 2015.

\bibitem[Magnani \& Boyd(2009)Magnani and Boyd]{magnani2009convex}
Magnani, Alessandro and Boyd, Stephen~P.
\newblock Convex piecewise-linear fitting.
\newblock \emph{Optimization and Engineering}, 10\penalty0 (1):\penalty0 1--17,
  2009.

\bibitem[Mnih et~al.(2015)Mnih, Kavukcuoglu, Silver, Rusu, Veness, Bellemare,
  Graves, Riedmiller, Fidjeland, Ostrovski, et~al.]{mnih2015human}
Mnih, Volodymyr, Kavukcuoglu, Koray, Silver, David, Rusu, Andrei~A, Veness,
  Joel, Bellemare, Marc~G, Graves, Alex, Riedmiller, Martin, Fidjeland,
  Andreas~K, Ostrovski, Georg, et~al.
\newblock Human-level control through deep reinforcement learning.
\newblock \emph{Nature}, 518\penalty0 (7540):\penalty0 529--533, 2015.

\bibitem[Nair \& Hinton(2010)Nair and Hinton]{nair2010rectified}
Nair, Vinod and Hinton, Geoffrey~E.
\newblock Rectified linear units improve restricted boltzmann machines.
\newblock In \emph{Proceedings of the 27th International Conference on Machine
  Learning (ICML-10)}, pp.\  807--814, 2010.

\bibitem[Oliphant(2006)]{oliphant2006guide}
Oliphant, Travis~E.
\newblock \emph{A guide to NumPy}, volume~1.
\newblock Trelgol Publishing USA, 2006.

\bibitem[Pedregosa et~al.(2011)Pedregosa, Varoquaux, Gramfort, Michel, Thirion,
  Grisel, Blondel, Prettenhofer, Weiss, Dubourg, et~al.]{pedregosa2011scikit}
Pedregosa, Fabian, Varoquaux, Ga{\"e}l, Gramfort, Alexandre, Michel, Vincent,
  Thirion, Bertrand, Grisel, Olivier, Blondel, Mathieu, Prettenhofer, Peter,
  Weiss, Ron, Dubourg, Vincent, et~al.
\newblock Scikit-learn: Machine learning in python.
\newblock \emph{The Journal of Machine Learning Research}, 12:\penalty0
  2825--2830, 2011.

\bibitem[Peng et~al.(2009)Peng, Bo, and Xu]{peng2009conditional}
Peng, Jian, Bo, Liefeng, and Xu, Jinbo.
\newblock Conditional neural fields.
\newblock In \emph{Advances in neural information processing systems}, pp.\
  1419--1427, 2009.

\bibitem[Polyak(1964)]{polyak1964some}
Polyak, Boris~T.
\newblock Some methods of speeding up the convergence of iteration methods.
\newblock \emph{USSR Computational Mathematics and Mathematical Physics},
  4\penalty0 (5):\penalty0 1--17, 1964.

\bibitem[Poon \& Domingos(2011)Poon and Domingos]{poon2011sum}
Poon, Hoifung and Domingos, Pedro.
\newblock Sum-product networks: A new deep architecture.
\newblock In \emph{{UAI} 2011, Proceedings of the Twenty-Seventh Conference on
  Uncertainty in Artificial Intelligence, Barcelona, Spain, July 14-17, 2011},
  pp.\  337--346, 2011.

\bibitem[Ratliff et~al.(2007)Ratliff, Bagnell, and
  Zinkevich]{ratliff2007approximate}
Ratliff, Nathan~D, Bagnell, J~Andrew, and Zinkevich, Martin.
\newblock ({A}pproximate) subgradient methods for structured prediction.
\newblock In \emph{International Conference on Artificial Intelligence and
  Statistics}, pp.\  380--387, 2007.

\bibitem[Rumelhart et~al.(1988)Rumelhart, Hinton, and
  Williams]{rumelhart1988learning}
Rumelhart, David~E, Hinton, Geoffrey~E, and Williams, Ronald~J.
\newblock Learning representations by back-propagating errors.
\newblock \emph{Cognitive modeling}, 5\penalty0 (3):\penalty0 1, 1988.

\bibitem[Samaria \& Harter(1994)Samaria and
  Harter]{samaria1994parameterisation}
Samaria, Ferdinando~S and Harter, Andy~C.
\newblock Parameterisation of a stochastic model for human face identification.
\newblock In \emph{Applications of Computer Vision, 1994., Proceedings of the
  Second IEEE Workshop on}, pp.\  138--142. IEEE, 1994.

\bibitem[Simard \& LeCun(1991)Simard and LeCun]{simard1991reverse}
Simard, Patrice and LeCun, Yann.
\newblock Reverse tdnn: an architecture for trajectory generation.
\newblock In \emph{Advances in Neural Information Processing Systems}, pp.\
  579--588. Citeseer, 1991.

\bibitem[Simonyan \& Zisserman(2014)Simonyan and Zisserman]{simonyan2014very}
Simonyan, Karen and Zisserman, Andrew.
\newblock Very deep convolutional networks for large-scale image recognition.
\newblock \emph{arXiv preprint arXiv:1409.1556}, 2014.

\bibitem[Smola et~al.(2008)Smola, Vishwanathan, and Le]{smola2007bundle}
Smola, Alex~J., Vishwanathan, S.v.n., and Le, Quoc~V.
\newblock Bundle methods for machine learning.
\newblock In Platt, J.~C., Koller, D., Singer, Y., and Roweis, S.~T. (eds.),
  \emph{Advances in Neural Information Processing Systems 20}, pp.\
  1377--1384. Curran Associates, Inc., 2008.

\bibitem[Szegedy et~al.(2015)Szegedy, Liu, Jia, Sermanet, Reed, Anguelov,
  Erhan, Vanhoucke, and Rabinovich]{szegedy2015going}
Szegedy, Christian, Liu, Wei, Jia, Yangqing, Sermanet, Pierre, Reed, Scott,
  Anguelov, Dragomir, Erhan, Dumitru, Vanhoucke, Vincent, and Rabinovich,
  Andrew.
\newblock Going deeper with convolutions.
\newblock In \emph{Proceedings of the IEEE Conference on Computer Vision and
  Pattern Recognition}, pp.\  1--9, 2015.

\bibitem[Taskar et~al.(2005)Taskar, Chatalbashev, Koller, and
  Guestrin]{taskar2005learning}
Taskar, Ben, Chatalbashev, Vassil, Koller, Daphne, and Guestrin, Carlos.
\newblock Learning structured prediction models: A large margin approach.
\newblock In \emph{Proceedings of the 22nd International Conference on Machine
  Learning}, pp.\  896--903. ACM, 2005.

\bibitem[Todorov et~al.(2012)Todorov, Erez, and Tassa]{todorov2012mujoco}
Todorov, Emanuel, Erez, Tom, and Tassa, Yuval.
\newblock Mujoco: A physics engine for model-based control.
\newblock In \emph{2012 IEEE/RSJ International Conference on Intelligent Robots
  and Systems}, pp.\  5026--5033. IEEE, 2012.

\bibitem[Tsochantaridis et~al.(2005)Tsochantaridis, Joachims, Hofmann, and
  Altun]{tsochantaridis2005large}
Tsochantaridis, Ioannis, Joachims, Thorsten, Hofmann, Thomas, and Altun,
  Yasemin.
\newblock Large margin methods for structured and interdependent output
  variables.
\newblock \emph{Journal of Machine Learning Research}, 6:\penalty0 1453--1484,
  2005.

\bibitem[Tsoumakas et~al.(2011)Tsoumakas, Spyromitros-Xioufis, Vilcek, and
  Vlahavas]{tsoumakas2011mulan}
Tsoumakas, Grigorios, Spyromitros-Xioufis, Eleftherios, Vilcek, Jozef, and
  Vlahavas, Ioannis.
\newblock Mulan: A java library for multi-label learning.
\newblock \emph{Journal of Machine Learning Research}, 12\penalty0
  (Jul):\penalty0 2411--2414, 2011.

\bibitem[Van~Rossum \& Drake~Jr(1995)Van~Rossum and Drake~Jr]{van1995python}
Van~Rossum, Guido and Drake~Jr, Fred~L.
\newblock \emph{Python reference manual}.
\newblock Centrum voor Wiskunde en Informatica Amsterdam, 1995.

\bibitem[Wright(1997)]{wright1997primal}
Wright, Stephen~J.
\newblock \emph{Primal-dual interior-point methods}.
\newblock Siam, 1997.

\bibitem[Zheng et~al.(2015)Zheng, Jayasumana, Romera-Paredes, Vineet, Su, Du,
  Huang, and Torr]{zheng2015conditional}
Zheng, Shuai, Jayasumana, Sadeep, Romera-Paredes, Bernardino, Vineet, Vibhav,
  Su, Zhizhong, Du, Dalong, Huang, Chang, and Torr, Philip~HS.
\newblock Conditional random fields as recurrent neural networks.
\newblock In \emph{Proceedings of the IEEE International Conference on Computer
  Vision}, pp.\  1529--1537, 2015.

\end{thebibliography}
\bibliographystyle{icml2017}

\newpage
\appendix
\twocolumn[
  \icmltitle{Input Convex Neural Networks: Supplementary Material}
  \begin{icmlauthorlist}
  \icmlauthor{Brandon Amos}{}
  \icmlauthor{Lei Xu}{}
  \icmlauthor{J.~Zico Kolter}{}
  \end{icmlauthorlist}
  \vskip 0.3in
]
\icmltitlerunning{Input Convex Neural Networks: Supplementary Material}
\section{Additional architectures}
\label{sec:additional-arch}

\subsection{Convolutional architectures}

Convolutions are important to many visual structured tasks.
We have left convolutions out to keep the prior ICNN notation
light by using matrix-vector operations.
ICNNs can be similarly created with convolutions because
the convolution is a linear operator.

The construction of convolutional layers in ICNNs depends
on the type of input and output space.
If the input and output space are similarly
structured (e.g. both spatial), the $j$th feature map
of a convolutional PICNN layer $i$ can be defined by
\begin{equation}
\begin{split}
z_{i+1}^j & = g_i\left(z_i\ast W_{i,j}^{(z)} + (Sx)\ast W_{i,j}^{(x)} + (Sy)\ast
W_{i,j}^{(y)} + b_{i,j} \right) \\
\end{split}
\end{equation}
where the convolution kernels $W$ are the same size and
$S$ scales the input and output to be the same size as
the previous feature map, and were we omit some of the Hadamard product terms
that can appear above for simplicity of presentation.

If the input space is spatial, but the output space has another structure
(e.g. the simplex), the convolution over the output space can
be replaced by a matrix-vector operation, such as

\begin{equation}
\begin{split}
z_{i+1}^j & = g_i\left(z_i\ast W_{i,j}^{(z)} + (Sx)\ast W_{i,j}^{(x)} + B_{i,j}^{(y)}y + b_{i,j} \right) \\
\end{split}
\end{equation}
where the product $B_{i,j}^{(y)}y$ is a scalar.

\section{Exact inference in ICNNs}
\label{sec:exact-inference}
Although it is not a practical approach for solving the optimization tasks, we
first highlight the fact that the inference problem for the networks
presented above (where the non-linear are either ReLU or linear units) can be
posed as as linear program.  Specifically, considering the FICNN network in
\eqref{eq-ficnn} can be written as the optimization problem
\begin{equation}
\begin{split}
\minimize_{y,z_1,\ldots,z_k} \;\; & z_k \\
 \subjectto \;\; & z_{i+1} \geq W^{(z)}_i z_i + W^{(y)}_i y + b_i,  \;\;
 i=0,\ldots,k-1 \\
& z_i \geq 0, \;\; i=1,\ldots,k-1.
\end{split}
\end{equation}
This problem exactly replicates the equations of the FICNN, with the exception
that we have replaced ReLU and the equality constraint between layers with a
positivity constraint on the $z_i$ terms and an inequality.  However, because we
are minimizing the final $z_k$ term, and because each inequality constraint is
convex, at the solution one of these constraints must be tight, i.e., $(z_i)_j
= (W^{(z)}_i z_i + W^{(y)}_i y + b_i)_j$ or $(z_i)_j = 0$, which recovers the
ReLU non-linearity exactly.  The exact same procedure can be used to write to
create an exact inference procedure for the PICNN.

Although the LP formulation is appealing in its simplicity, in practice these
optimization problems will have a number of variables equal to the \emph{total}
number of activations in the entire network.  Furthermore, most LP solution
methods to solve such problems require that we form \emph{and invert}
structured matrices with blocks such as $W_i^T W_i$
--- the case for most interior-point methods \citep{wright1997primal} or
even approximate algorithms such as the
alternating direction method of multipliers \citep{boyd2011distributed} ---
which are large dense matrices or have structured forms such as
non-cyclic convolutions that are expensive to invert.
Even incremental approaches like the Simplex method
require that we form inverses of subsets of columns of these matrices, which are
additionally different for structured operations like convolutions, and which
overall still involve substantially more computation than a single forward pass.
Furthermore, such solvers typically do not exploit the substantial effort that
has gone in to accelerating the forward and backward computation passes for
neural networks using hardware such as GPUs.  Thus, as a whole, these do not
present a viable option for optimizing the networks.

\newpage
\section{The bundle method for approximate inference in ICNNs}
\label{sec:bundle}
We here review the basic bundle method \citep{smola2007bundle} that we build
upon in our bundle entropy method.
The bundle method takes
advantage of the fact that for a convex objective, the first-order approximation
at any point is a global \emph{underestimator} of the function; this lets us
maintain a piecewise linear lower bound on the function by adding
cutting planes formed by this first order approximation, and then repeatedly
optimizing this lower bound.  Specifically, the process follows the procedure
shown in Algorithm~\ref{alg:bundle}. Denoting the iterates of the algorithm as
$y^k$, at each iteration of the algorithm, we compute the first order
approximation to the function
\begin{equation}
f(x, y^k;\theta) + \nabla_y f(x, y^k;\theta)^T (y - y^k)
\end{equation}
and update the next iteration by solving the optimization problem
\begin{equation}
y^{k+1} := \argmin_{y \in \mathcal{Y}} \max_{1\leq i\leq k} \{f(x, y^i;\theta) + \nabla_y f
(x, y^i;\theta)^T (y - y^i)\}.
\end{equation}
A bit more concretely, the optimization problem can be written via a set of linear
inequality constraints
\begin{equation}
y^{k+1},t^{k+1} := \argmin_{y \in \mathcal{Y}, t} \;\; \{t \mid G y + h \leq
t1\}
\end{equation}
where $G \in \mathbb{R}^{k \times n}$ has rows equal to
\begin{equation}
g_i^T = \nabla_y f (x, y^i;\theta)^T
\end{equation}
and $h \in \mathbb{R}^k$ has entries equal to
\begin{equation}
h_i = f(x, y^i;\theta) - \nabla_y f (x, y^i;\theta)^T y^i.
\end{equation}

\begin{algorithm}[t]
  \caption{A typical bundle method to optimize $f: \R^{m\times n}\rightarrow \R$
    over $\R^n$ for $K$ iterations with a fixed $x$ and initial starting point $y^1$.}
  \begin{algorithmic}[0]
    \Function{BundleMethod}{$f$, $x$, $y^1$, $K$}
    \Let{$G$}{$0\in\R^{K\times n}$}
    \Let{$h$}{$0\in\R^{K}$}
    \For{$k = 1,K$}
    \Let{$G_k^T$}{$\nabla_y f(x, y^k; \theta)^T$} \Comment{$k$th row of $G$}
    \Let{$h_k$}{$f(x, y^k; \theta) - \nabla_y f(x, y^k; \theta)^Ty^k$}
    \Let{$y^{k+1}$, $t^{k+1}$}{
      $\argmin_{y \in \mathcal{Y}, t} \;\; \{t \mid G_{1:k} y + h_{1:k} \leq t1\}$}
    \EndFor
    \State \Return $y^{K+1}$
    \EndFunction
  \end{algorithmic}
  \label{alg:bundle}
\end{algorithm}

\section{Bundle Entropy Algorithm}
\label{sec:bundle-entropy-alg}
In Algorithm~\ref{alg:bundle-entropy}.
\begin{algorithm}[t]
  \caption{Our bundle entropy method to optimize $f: \R^m\times [0,1]^n\rightarrow \R$
    over $[0,1]^n$
    for $K$ iterations with a fixed $x$ and initial starting point $y^1$.
  }
  \begin{algorithmic}[0]
    \Function{BundleEntropyMethod}{$f$, $x$, $y^1$, $K$}
    \Let{$G_\ell$}{$[\ ]$}
    \Let{$h_\ell$}{$[\ ]$}
    \For{$k = 1,K$}
    \State \Call{Append}{$G_\ell$, $\nabla_y f(x, y^k; \theta)^T$}
    \State \Call{Append}{$h_\ell$, $f(x, y^k; \theta) - \nabla_y f(x, y^k; \theta)^Ty^k$}
    \Let{$a_k$}{\Call{Length}{$G_\ell$}}\Comment{The number of active constraints.}
    \Let{$G_k$}{\Call{Concat}{$G_\ell$}$\in\R^{a_k\times n}$}
    \Let{$h_k$}{\Call{Concat}{$h_\ell$}$\in\R^{a_k}$}
    \If{$a_k = 1$}
    \Let{$\lambda_k$}{1}
    \Else
    \Let{$\lambda_k$}{\Call{ProjNewtonLogistic}{$G_k$, $h_k$}}
    \EndIf
    \Let{$y^{k+1}$}{$(1+\exp(G_k^T\lambda_k))^{-1}$}

    \State \Call{Delete}{$G_\ell[i]$ and $h_\ell[i]$ where $\lambda_i \leq 0$}
    \Comment{Prune inactive constraints.}
    \EndFor
    \State \Return $y^{K+1}$
    \EndFunction
  \end{algorithmic}
  \label{alg:bundle-entropy}
\end{algorithm}

\section{Deep Q-learning with ICNNs}
\label{sec:icnn-rl}
In Algorithm~\ref{alg:icnn-rl}.

\begin{algorithm}[t]
  \caption{Deep Q-learning with ICNNs.
    \texttt{Opt-Alg} is a convex minimization algorithm such as
    gradient descent or the bundle entropy method.
    $\tilde Q_\theta$ is the objective the optimization algorithm solves.
    In gradient descent, $\tilde Q_\theta(s,a) = Q(s, a|\theta)$ and
    with the bundle entropy method, $\tilde Q_\theta(s,a) = Q(s, a|\theta) + H(a)$.
  }
  \begin{algorithmic}
    \State{Select a discount factor $\gamma\in(0,1)$ and moving average factor
      $\tau\in(0,1)$}
    \State{Initialize the ICNN $-Q(s, a|\theta)$ with
      target network parameters $\theta'\leftarrow\theta$
      and a replay buffer $R\leftarrow\emptyset$}
    \For{each episode $e=1,E$}
    \State{Initialize a random process $\mathcal{N}$ for action exploration}
    \State{Receive initial observation state $s_1$}
    \For{$i=1,I$}
    \Let{$a_i$}{\Call{Opt-Alg}{$-Q_\theta$, $s_i$, $a_{i,0}$}+$\mathcal{N}_i$}
    \Comment{For some initial action $a_{i,0}$}
    \State{Execute $a_i$ and observe $r_{i+1}$ and $s_{i+1}$}
    \State \Call{Insert}{$R$, $(s_i, a_i, s_{i+1}, r_{i+1})$}
    \State{Sample a random minibatch from the replay buffer: $R_M\subseteq R$}
    \For{$(s_m, a_m, s_m^+, r_m^+)\in R_M$}
    \Let{$a_m^+$}{\Call{Opt-Alg}{$-Q_{\theta'}$,$s_{m}^+$,$a_{m,0}^+$}}
    \Comment{Uses the target parameters $\theta'$}
    \Let{$y_m$}{$r_m^+ + \gamma Q(s_m^+, a_m^+|\theta')$}
    \EndFor
    \State{Update $\theta$ with a gradient step to minimize
      $\mathcal{L} = \frac{1}{|R_M|}\sum_m\big(\tilde Q(s_m, a_m|\theta)-y_m\big)^2$}
    \Let{$\theta'$}{$\tau\theta + (1-\tau)\theta'$}
    \Comment{Update the target network.}
    \EndFor
    \EndFor
  \end{algorithmic}
  \label{alg:icnn-rl}
\end{algorithm}

\newpage
\section{Max-margin structured prediction}
\label{sec:max-margin}

In the more traditional structured prediction setting, where we do not aim to
fit the energy function directly but fit the predictions made by the system to
some target outputs, there are different possibilities for learning the ICNN
parameters.  One such method is based upon the max-margin structured
prediction framework \citep{tsochantaridis2005large,taskar2005learning}.  Given
some training example $(x, y^\star)$, we would like to require that this example
has a joint energy that is lower than all other possible values for $y$.  That
is, we want the function $\tilde{f}$ to satisfy the constraint
\begin{equation}
\tilde{f}(x, y^\star;\theta) \leq \min_y \tilde{f}(x,y;\theta)
\end{equation}
Unfortunately, these conditions can be trivially fit by choosing a constant
$\tilde{f}$ (although the entropy term alleviates this problem slightly, we can
still choose an approximately constant function), so instead the max-margin
approach adds a margin-scaling term that requires this gap to be larger for $y$
further from $y^\star$, as measured by some loss function $\Delta(y,y^\star)$.
Additionally adding slack variables to allow for potential violation of these
constraints, we arrive at the typical max-margin structured prediction
optimization problem
\begin{equation}
\begin{split}
\minimize_{\theta,\xi \geq 0} \;\; & \frac{\lambda}{2}\|\theta\|_2^2 + \sum_
{i=1}^m
\xi_i \\
\subjectto \;\; &\tilde{f}(x_i,y_i;\theta) \leq \min_{y \in \mathcal{Y}} \left
(\tilde{f}(x_i,y;\theta) - \Delta(y_i,y) \right ) - \xi_i
\end{split}
\label{eq:structured-prediction}
\end{equation}
As a simple example, for multiclass classification tasks where $y^\star$ denotes
a ``one-hot'' encoding of examples, we can use a multi-variate entropy term and
let $\Delta (y,y^\star) = {y^\star}^T (1 - y)$. Training
requires solving this ``loss-augmented'' inference problem, which is convex
for suitable choices of the margin scaling term.

The optimization problem \eqref{eq:structured-prediction} is naturally still
\emph{not convex} in $\theta$, but can be solved via the subgradient method
for structured prediction \citep{ratliff2007approximate}.  This algorithm
iteratively selects a training example $x_i, y_i$, then 1) solves the
optimization problem
\begin{equation}
y^\star = \argmin_{y \in \mathcal{Y}} f(x_i,y;\theta) - \Delta(y_i,y)
\end{equation}
and 2) if the margin is violated, updates the network's parameters according
to the subgradient
\begin{equation}
\theta := \mathcal{P}_+\left [ \theta - \alpha \left(
                \lambda\theta +
                \nabla_\theta f(x_i,y_i,\theta) -
                \nabla_\theta f (x_i, y^\star;\theta)\right)\right ]
\end{equation}
where $\mathcal{P}_+$ denotes the projection of $W^{(z)}_{1:k-1}$ onto the non-negative
orthant. This method can be easily adapted to use mini-batches instead of a
single example per subgradient step, and also adapted to alternative optimization
methods like AdaGrad \citep{duchi2011adaptive} or ADAM \citep{kingma2014adam}.
Further, a fast approximate solution to $y^\star$ can be used instead
of the exact solution.


\section{Proof of Proposition \ref{proposition-gradient}}
\label{sec:argmin-diff-proof}
\begin{proof}[Proof (of Proposition \ref{proposition-gradient}).]
We have by the chain rule that
\begin{equation}
\frac{\partial \ell }{\partial \theta} =
\frac{\partial \ell}{\partial \hat{y}} \left
( \frac{\partial \hat{y}}{\partial G} \frac{\partial G}{\partial \theta} +
\frac{\partial \hat{y}}{\partial h} \frac{\partial h}{\partial \theta}
\right).
\end{equation}
The challenging terms to compute in this equation are the $\frac{\partial \hat
{y}}
{\partial G}$ and $\frac{\partial \hat{y}}{\partial h}$ terms.  These can be
computed (although we will ultimately not compute them explicitly, but just
compute the product of these matrices and other terms in the Jacobian), by
implicit differentiation of the KKT conditions.  Specifically, the
KKT conditions of the bundle entropy method (considering only the active
constraints at the solution) are given by
\begin{equation}
\begin{split}
1 + \log \hat{y} - \log (1-\hat{y}) + G^T \lambda & = 0 \\
G\hat{y} + h - t1 & = 0 \\
1^T \lambda & = 1.
\end{split}
\end{equation}
For simplicity of presentation, we consider first the Jacobian with respect to
$h$.  Taking differentials of these equations with respect to $h$ gives
\begin{equation}
\begin{split}
\diag\left(\frac{1}{\hat{y}} + \frac{1}{1-\hat{y}}\right) \dd y + G^T \dd
\lambda & = 0 \\
G \dd y + \dd h - \dd t 1 & = 0 \\
1^T \dd \lambda & = 0
\end{split}
\end{equation}
or in matrix form
\begin{equation}
\left [ \begin{array}{ccc}
\diag\left(\frac{1}{\hat{y}} + \frac{1}{1-\hat{y}}\right) & G^T & 0 \\
G & 0 & -1 \\
0 & -1^T & 0 \end{array} \right ]
\left [ \begin{array}{c} \dd y \\ \dd \lambda \\ \dd t \end{array} \right ]
= \left [ \begin{array}{c} 0 \\ -\dd h \\ 0 \end{array} \right ].
\end{equation}
To compute the Jacobian $\frac{\partial \hat{y}}{\partial h}$ we can solve the
system above with the right hand side given by $\dd h = I$, and the resulting
$\dd y$ term will be the corresponding Jacobian.  However, in our ultimate
objective we always left-multiply the proper terms in the above equation by
$\frac{\partial \ell}{\partial \hat{y}}$.  Thus, we instead define
{\small
\begin{equation}
\left [ \begin{array}{c} c^y \\ c^\lambda \\ c^t \end{array} \right ]
 = \left [ \begin{array}{ccc}
\diag\left(\frac{1}{\hat y} + \frac{1}{1-\hat y}\right) & G^T & 0 \\
G & 0 & -1 \\
0 & -1^T & 0 \end{array} \right ]^{-1}
\left [ \begin{array}{c} -(\frac{\partial \ell}{\partial \hat y})^T \\ 0 \\ 0
\end{array} \right ]
\end{equation}
}
and we have the the simple formula for the Jacobian product
\begin{equation}
\frac{\partial \ell}{\partial \hat{y}} \frac{\partial \hat{y}}{\partial h} =
(c^\lambda)^T.
\end{equation}

A similar set of operations taking differentials with respect to $G$ leads to
the matrix equations
{\small
\begin{equation}
\left [ \begin{array}{ccc}
\diag\left(\frac{1}{\hat y} + \frac{1}{1-\hat y}\right) & G^T & 0 \\
G & 0 & -1 \\
0 & -1^T & 0 \end{array} \right ]
\left [ \begin{array}{c} \dd y \\ \dd \lambda \\ \dd t \end{array} \right ]
= \left [ \begin{array}{c} -\dd G^T \lambda \\ - \dd G y \\ 0 \end{array}
\right ]
\end{equation}
}
and the corresponding Jacobian products / gradients are given by
\begin{equation}
\frac{\partial \ell}{\partial \hat{y}} \frac{\partial \hat{y}}{\partial
G} = c^y \lambda^T + \hat{y} (c^\lambda)^T.
\end{equation}
Finally, using the definitions that
\begin{equation}
g_i^T = \nabla_y f (x, y^i;\theta)^T,\;\;  h_i = f(x, y^k;\theta) -
\nabla_y f(x,y^i;\theta)^T y^i
\end{equation}
we recover the formula presented in the proposition.
\end{proof}

\section{State and action space sizes in the OpenAI gym MuJoCo benchmarks.}
\label{sec:gym-szs}
\begin{table}[H]
  \centering
  \begin{tabular}{rrr}
    Environment & \# State & \# Action \\ \hline
    InvertedPendulum-v1	& 4 & 1 \\
    InvertedDoublePendulum-v1	& 11 & 1 \\
    Reacher-v1	& 11 & 2 \\
    HalfCheetah-v1	& 17 & 6 \\
    Swimmer-v1	& 8 & 2 \\
    Hopper-v1	& 11 & 3 \\
    Walker2d-v1	& 17 & 6 \\ 
    Ant-v1	& 111 & 8 \\
    Humanoid-v1	& 	376 & 17 \\
    HumanoidStandup-v1	& 	376 & 17 \\
  \end{tabular}
  \caption{State and action space sizes in the OpenAI gym MuJoCo benchmarks.}
  \label{tab:gym-szs}
\end{table}

\section{Synthetic classification examples}
\label{sec:synthetic}
\begin{figure*}[t]
  \centering
  \begin{minipage}{0.6\textwidth}
    \includegraphics[width=\textwidth]{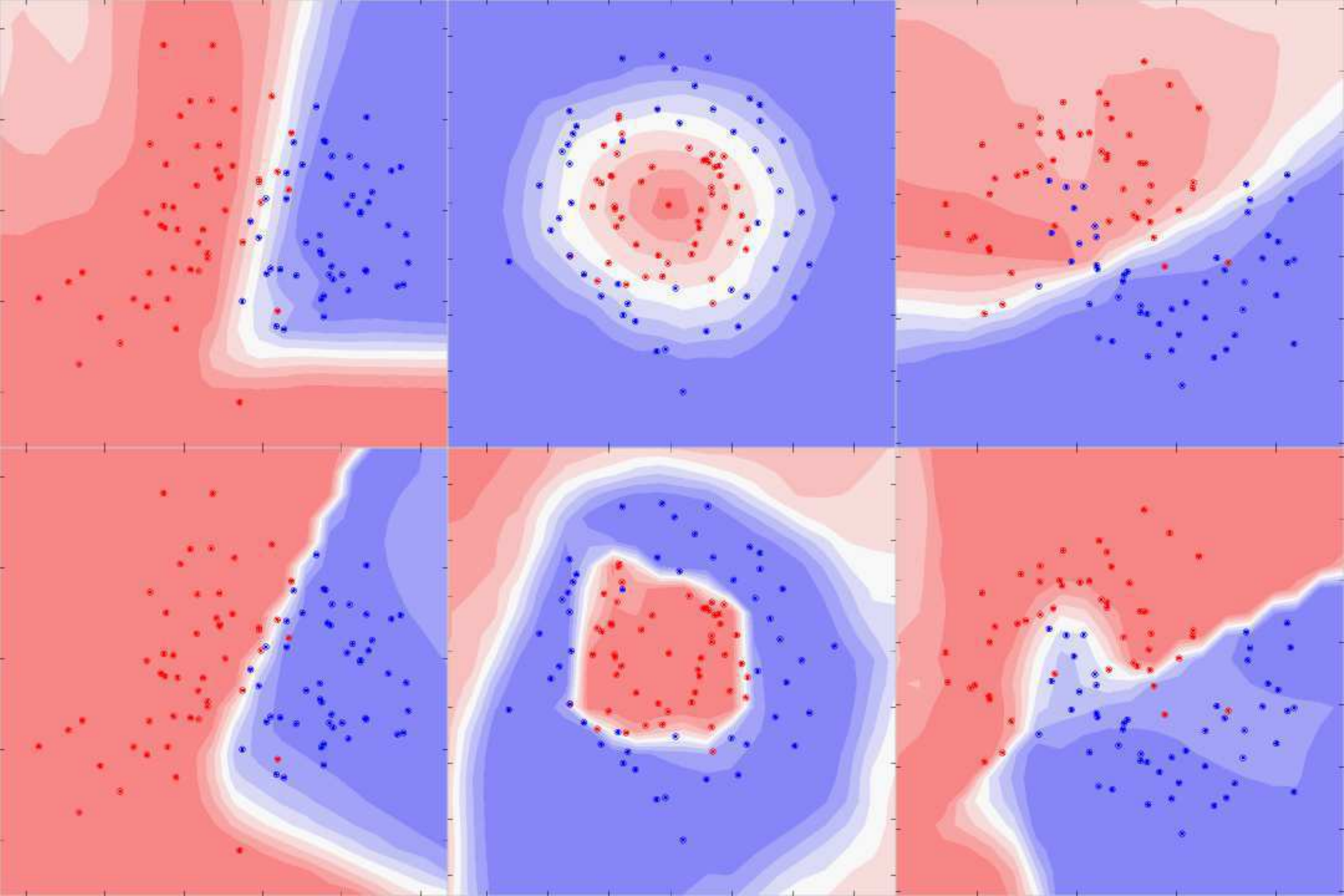}
  \end{minipage}
  \begin{minipage}{10mm}
    \includegraphics[width=10mm]{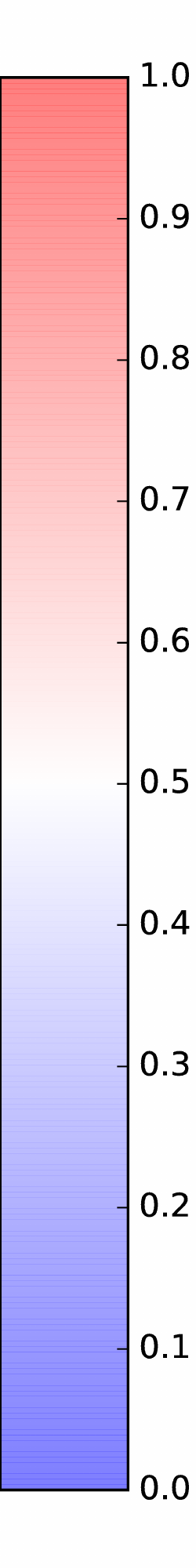}
  \end{minipage}
  \caption{
    FICNN (top) and PICNN (bottom) classification of synthetic non-convex
    decision boundaries. Best viewed in color.
  }
  \label{fig:exp:synthetic}
\end{figure*}

We begin with a simple example to illustrate the classification performance of a
two-hidden-layer FICNN and PICNN on two-dimensional binary classification
tasks from the scikit-learn toolkit \citep{pedregosa2011scikit}.
Figure~\ref{fig:exp:synthetic} shows the classification performance on the dataset.
The FICNN's energy function which is fully convex in $\cal{X}\times\cal{Y}$ jointly is
able to capture complex, but sometimes restrictive decision boundaries.
The PICNN, which is nonconvex over $\cal{X}$ but convex over $\cal{Y}$
overcomes these restrictions and can capture more complex decision boundaries.

\section{Multi-Label Classification Training Plots}
\label{sec:exp:ml:f1}
In Figure~\ref{fig:exp:ml:f1}.
\begin{figure*}[t]
  \centering
  \includegraphics[width=0.35\textwidth]{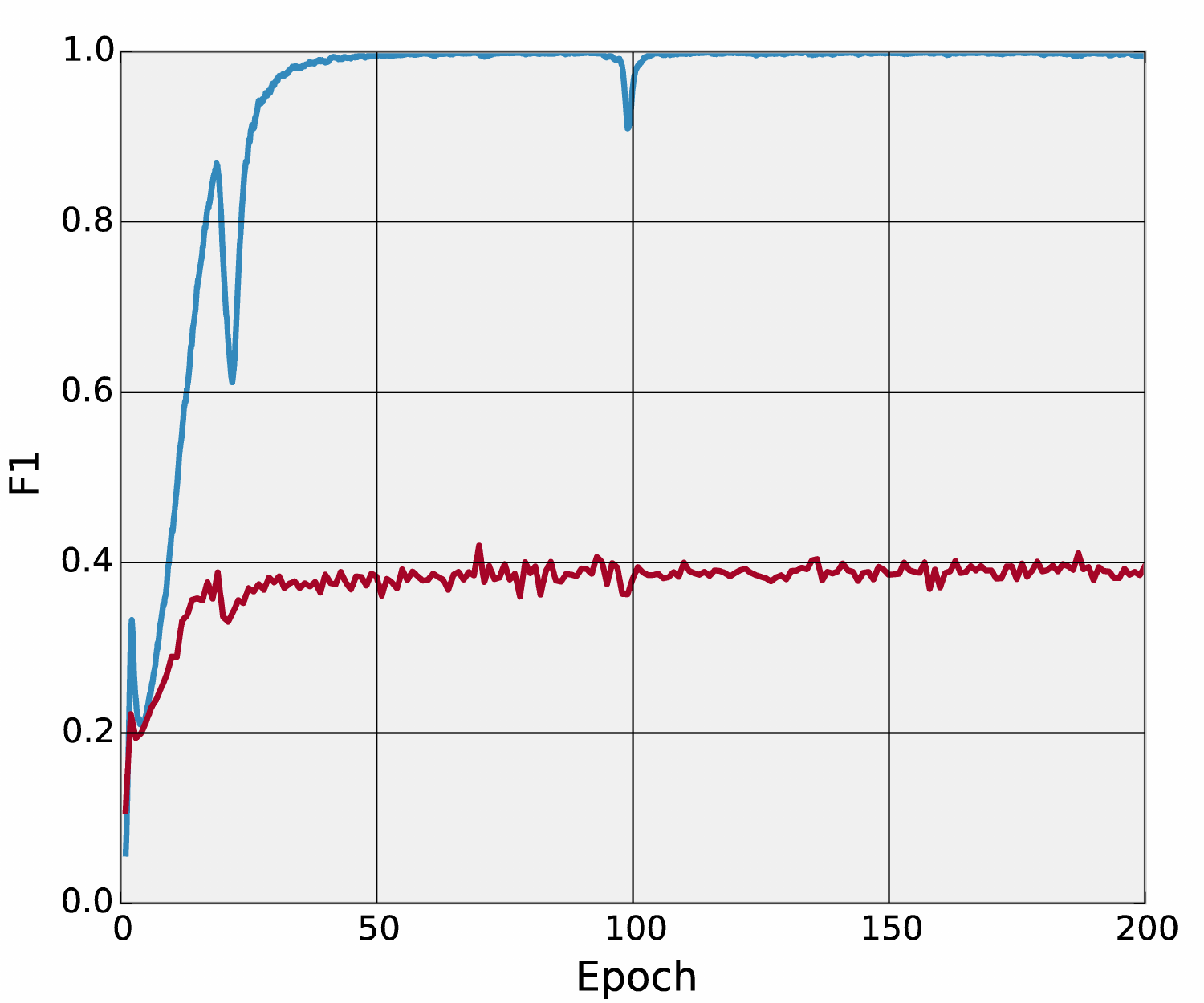}
  \includegraphics[width=0.35\textwidth]{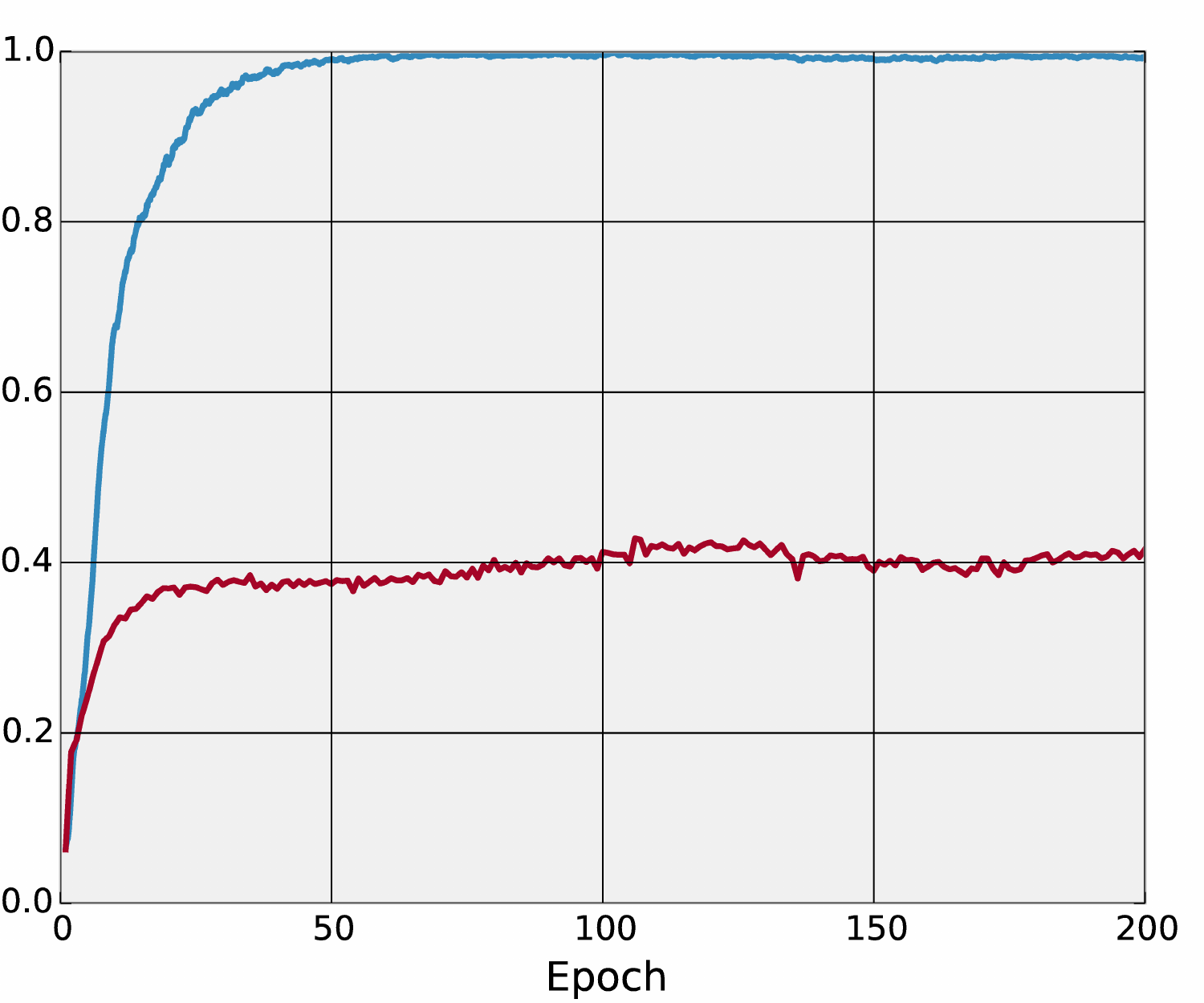}
  \caption{
    Training (blue) and test (red) macro-F1 score of
    a feedforward network (left) and PICNN (right) on the BibTeX
    multi-label classification dataset.
    The final test F1 scores are 0.396 and 0.415, respectively.
    (Higher is better.)
  }
  \label{fig:exp:ml:f1}
\end{figure*}

\section{Image Completion}
\label{sec:appendix-completion}

The losses are in Figure~\ref{fig:exp:olivetti:loss}.

\begin{figure*}[t]
  \centering
  \includegraphics[width=\textwidth]{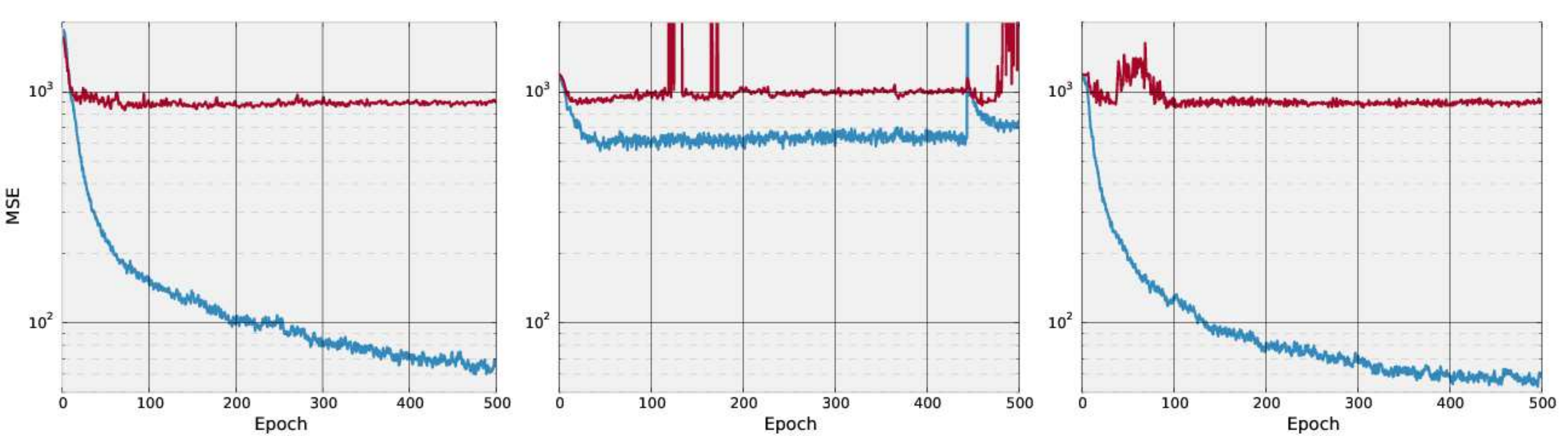}
  \caption{Mean Squared Error (MSE) on the train
    (blue, rolling over 1 epoch) and test (red) images
    from Olivetti faces for PICNNs trained with the
    bundle entropy method (left) and back optimization (center),
    and back optimization with the convexity constraint relaxed (right).
    The minimum test MSEs are 833.0, 872.0, and 850.9, respectively.
  }
  \label{fig:exp:olivetti:loss}
\end{figure*}

\end{document}